\documentclass[11pt,lettersize]{article}

\usepackage{times}

\usepackage[usenames,dvipsnames]{xcolor}

\usepackage{amsmath,amssymb,bbm}
\usepackage{fullpage}
\usepackage{amsthm}
\usepackage{pdfpages}
\usepackage{comment}
\usepackage{hyperref}

\usepackage[mathscr]{euscript}
\usepackage{mathtools}
\usepackage[portrait,letterpaper,margin=1in]{geometry}
\usepackage{verbatim}
\usepackage{subcaption}

\usepackage[utf8]{inputenc}
\usepackage{microtype}
\usepackage{algorithm}
\usepackage{algpseudocode}

\newcommand{\ljs}[1]{{ \textcolor{red}{(LJS:  #1)}}{}}

\newcommand{\tildecal}[1]{\tilde{\mathcal{#1}}}

\theoremstyle{plain}
\newtheorem{theorem}{Theorem}
\newtheorem{lemma}[theorem]{Lemma}

\newtheorem{corollary}[theorem]{Corollary}

\newtheorem{claim}[theorem]{Claim}
\theoremstyle{definition}
\newtheorem{definition}[theorem]{Definition}

\newcommand{\One}{\mathbbm{1}}

\newenvironment{proofof}[1]{ {\noindent \em Proof of #1.}\/}{\hfill\qedsymbol\bigskip}
\newenvironment{proofsketch}{ {\noindent \em Proof sketch.}\/}{\hfill\qedsymbol\bigskip}

\newcommand{\remove}[1]{}
\newcommand{\suppress}[1]{}

\newcommand{\RR}{{\mathbb{R}}}

\newcommand{\argmin}{{\mathrm{argmin}}}

\newcommand{\eps}{\epsilon}

\DeclareMathOperator{\op}{op}

\def\Real{\mathbb{R}}
\def\Complex{\mathbb{C}}
\def\C{\Complex}
\def\R{\Real}
\DeclareMathOperator{\diag}{diag}

\def\abs#1{\mathopen| #1 \mathclose|}			

\def\Prod{\prod\limits}

\def\Set#1{\left\{ #1 \right\}}
\def\Abs#1{\left| #1 \right|}

\def\Norm#1{\left\| #1 \right\|}
\def\Paren#1{\left( #1 \right)}		

\makeatletter
\def\Bigbar#1{\mathrel{\left|\vphantom{#1}\right.\n@space}}
\def\Setbar#1#2{\Set{#1 \Bigbar{#1 #2} #2}}

\makeatother

\def\tpose{\mathsf{T}}

\def\eps{\varepsilon}

\DeclareMathOperator{\supp}{supp}

\DeclareMathOperator{\Binomial}{Binomial}
\DeclareMathOperator{\Pas}{Pas}
\DeclareMathOperator{\Frob}{F}
\DeclareMathOperator{\poly}{poly}
\DeclareMathOperator{\Proj}{Proj}

\newcommand*\diff{\mathop{}\!\mathrm{d}}

\DeclareMathOperator{\cond}{cond}
\DeclareMathOperator{\roots}{roots}

\newcommand{\be}{\begin{equation}}
\newcommand{\ee}{\end{equation}}
\newcommand{\bea}{\begin{eqnarray}}
\newcommand{\eea}{\end{eqnarray}}
\newcommand{\bean}{\begin{eqnarray*}}
\newcommand{\eean}{\end{eqnarray*}}

\begin{document}

\title{The Sparse Hausdorff Moment Problem,
\\ with Application to Topic Models}

\author{
Spencer L. Gordon\thanks{Engineering and Applied Science, California Institute of Technology, {\tt slgordon@caltech.edu}.} \and
Bijan Mazaheri\thanks{Engineering and Applied Science, California Institute of Technology, {\tt bmazaher@caltech.edu}} \and
Yuval Rabani\thanks{The Rachel and Selim Benin School of Computer Science and Engineering, The Hebrew University of Jerusalem, Jerusalem 9190416, Israel, {\tt yrabani@cs.huji.ac.il}. Research supported in part by NSFC-ISF grant 2553-17 and by NSF-BSF grant 2018687. Part of this work was done while visiting Caltech.} \and
Leonard J. Schulman\thanks{Engineering and Applied Science, California Institute of Technology, {\tt schulman@caltech.edu}. Research supported in part by NSF grants CCF-1618795, 1909972.}
}

\maketitle

\begin{abstract}
We consider the problem of identifying, from its first $m$ noisy
moments, a probability distribution on $[0,1]$ of support $k<\infty$.
This is equivalent to the problem of learning a distribution on $m$
observable binary random variables $X_1,X_2,\dots,X_m$ that are iid
conditional on a hidden random variable $U$ taking values in
$\{1,2,\dots,k\}$. Our focus is on accomplishing this with $m=2k$,
which is the minimum $m$ for which verifying that the source is a
$k$-mixture is possible (even with exact statistics). This problem, so
simply stated, is quite useful: e.g., by a known reduction, any algorithm
for it lifts to an algorithm for learning pure topic models.

We give an algorithm for identifying a $k$-mixture using samples of 
$m=2k$ iid binary random variables using a sample of size 
$\left(1/w_{\min}\right)^2 \cdot\left(1/\zeta\right)^{O(k)}$
and post-sampling runtime of only $O(k^{2+o(1)})$ arithmetic operations.
Here $w_{\min}$ is the minimum probability of an outcome of $U$, and 
$\zeta$ is the minimum separation between the distinct success probabilities
of the $X_i$s. Stated in terms of the moment problem, it suffices to know 
the moments to additive accuracy $w_{\min}\cdot\zeta^{O(k)}$. It is known that 
the sample complexity of any solution to the identification problem must be 
at least exponential in $k$. Previous results demonstrated either worse
sample complexity and worse $O(k^c)$ runtime for some $c$ substantially 
larger than $2$, or similar sample complexity and much worse $k^{O(k^2)}$ 
runtime.

\end{abstract}

\thispagestyle{empty}
\newpage
\setcounter{page}{1}

\section{Introduction}\label{sec: intro}

\paragraph{Motivation.}
The Hausdorff moment problem is that of determining what moment sequences
\be \mu_i \coloneqq \int_0^1 \alpha^i \diff\mathcal{P}(\alpha) \quad \quad \quad (i \geq 0)
\label{moments-hausdorff} \ee
are possible for a probability distribution $\mathcal{P}$ supported on $[0,1]$. For background on this classical topic, see~\cite{Hildebrand74,Simon15}. 
Associated with this is the computational problem of determining $\mathcal{P}$ from $(\mu_i)_i$ (or approximating it from finite, and possibly noisy, prefixes). 

In this paper we are concerned with the sparse version of the computational problem, namely, the task of computing $\mathcal{P}$ in case it is assumed to have support of cardinality at most $k<\infty$. This problem has the following equivalent interpretation, due to which we call it the ``$k$-coin problem:''
identify the parameters of a distribution on $m$ observable binary random 
variables $X_1,X_2,\dots,X_m$ that are iid conditional on a 
hidden random variable $U$ taking values in $\{1,2,\dots,k\}$. 
Due to the symmetry among the variables (coins), the information available is precisely empirical estimates of moments $0,\ldots,m$ of the (shared) distribution of the variables $X_i$. (Of course $\mu_0=1$ so there are $m$ nontrivial statistics.)

Our focus is on accomplishing this with $m=2k$, which is the
minimum $m$ for which verifying that the source is a $k$-mixture is possible (even
with exact statistics). We freely go back and forth in this paper between the formulation in terms of mixture models and the formulation in terms of the moment problem. 

The problem of reconstructing $\mathcal{P}$, so simply stated, is quite useful: 

(i) By a known reduction, any algorithm for this problem lifts to an algorithm for learning topic models. This will be discussed in Sec.~\ref{sec: topic}.

(ii) This problem is a special case of the problem of identifying mixture models of $k$ product distributions on binary variables, a problem on which there has been an impressive sequence of contributions in the last two decades, as we will discuss below. The best runtime for the product-distributions problem is however $m^{O(k^2)}$. It seems likely that the true complexity of the general product case may be $m^{O(k)}$. Our results for the iid case may indicate a new direction toward resolving the conjecture for the general case.

(iii) Algorithms for identifying mixtures of product distributions are the simplest case of the yet-more-general problem of identifying distributions on ``structural causal models''~\cite{Pea09}. There has been little work in this direction,~\cite{AHK12} being a notable exception. However, even that work has to make strong assumptions about the distributions of the variables $X_i$, and in particular they cannot be binary. (Except for the case $k=2$, but we are concerned here with complexity of the problem as a function of $k$.) Source identification in causal graphical models is an important direction for future research, and the dependence of the sample size and runtime complexity on $k$ (a measure of how much ``confounding'' there is in the model) will matter a great deal.

\paragraph{The method.}
The algorithm we analyze is essentially that of Prony, 1795~\cite{Prony1795}. The idea is to (a) characterize the coin biases (the support of $\mathcal{P}$) as the roots of a polynomial whose coefficient vector is the kernel of the Hankel matrix; (b) use polynomial root-finding to determine the empirical coin biases; (c) reconstruct the mixture weights by polynomial interpolation. 

\paragraph{Prior work on the sparse Hausdorff, i.e., $k$-coin mixture, problem.}
It has long been acknowledged in the numerical analysis literature (e.g.,~\cite{Hildebrand74} 
\S 9.4,~\cite{KumaresanTS84}) that the Prony method is sensitive to sample error (i.e., to errors 
in the moments). In fact, the instability is not limited to the Prony method; a lower bound  
(\cite{RSS14} Thm~6.1) for the \textit{problem} (source identification of $k$-coin mixture models) 
shows that even for any $c<\infty$, if $ck$ (rather than just the minimum $2k$) noisy moments are 
available, it remains the case that accurate source identification is possible only if those moments 
are available to accuracy $\exp(-k)$, i.e., the sample size must be $\exp(k)$. The prior results on 
upper bounds were these: (i)~\cite{RSS14} (within a paper devoted mostly to topic models) 
re-invented the Prony method unawares, and solved the problem using sample complexity 
$s = \max\{(1/\zeta)^{O(k)},k^{O(k^2)}\}$ (or moment accuracy $\min\{\zeta^{O(k)},(1/k)^{O(k^2)}\}$) 
and post-sampling runtime $\poly(k)$. (ii) A quite different algorithm in~\cite{LRSS15} improved the 
sample complexity to $k^{O(k)}$ (or the equivalent moment accuracy), but required post-sampling runtime 
$k^{O(k^2)}$.\footnote{The bounds in~\cite{RSS14,LRSS15} do not depend on $w_{\min}$ because accuracy of the output is measured in transportation norm, which is insensitive to small mixture weights.}
(iii) Motivated by a problem in population genetics that reduces to the
$k$-coin mixture problem,~\cite{KKMMR18} analyzed a solution using the Matrix Pencil Method,
which requires sample complexity $\left(1/w_{\min}\right)^4 \cdot\left(1/\zeta\right)^{O(k)}$. They don't discuss 
explicitly the post-sampling runtime complexity. The method requires solving a generalized eigenvalue problem, 
which is solved in practice using algorithms that run in time $O(k^3)$.\footnote{It is possible that the runtime
can be improved to the time it takes to multiply two $k\times k$ matrices. This is still much worse than
$O(k^2)$, and the best guarantees hide impractical constants.}

In this paper we simultaneously achieve sample complexity $\left(1/w_{\min}\right)^2 \cdot\left(1/\zeta\right)^{O(k)}$ 
and post-sampling runtime $O(k^{2+o(1)})$.

\paragraph{Our result.}
Our main result is that using the Prony method, source identification can be performed for $k$-coin 
mixtures with \textbf{sample complexity \boldmath $\left(1/w_{\min}\right)^2 \cdot\left(1/\zeta\right)^{O(k)}$} 
(equivalently, the required moment accuracy is $w_{\min}^2 \cdot\zeta^{O(k)}$), and 
\textbf{runtime \boldmath $k^{2+o(1)}$.} We have posted a working implementation of the algorithm on 
the following public Jupyter Notebook:  \href{https://colab.research.google.com/drive/1qR6VOYSjq08LPxqHhyY0ap_VL1apt9yS?usp=sharing}{\color{purple}{Online notebook implementation}}.\footnote{\url{https://colab.research.google.com/drive/1qR6VOYSjq08LPxqHhyY0ap_VL1apt9yS?usp=sharing}} (Tested in Chrome and Safari.) 
 
The dependence on a separation parameter between the coin biases is necessary since the mixture weights 
are ill-defined in the limit of coinciding coin biases. For the reader interested in the key technical novelties of 
the paper we might point to the quantitative characterization of the pseudo-kernel of a Hankel matrix in 
Lemma~\ref{lm: eigenvalue} and the sequence of error-control lemmas in Section~\ref{sec: analysis} and 
particularly to Lemma~\ref{lm: root stability} which shows why a pseudo-kernel-vector of the empirical Hankel 
matrix will, as a polynomial, have roots close to those of the kernel of the model Hankel matrix.

This result also implies an improvement in identifying pure topic models, via the reductions in~\cite{RSS14,LRSS15}. 
These reductions require solving $k$ binary instances, and the required accuracy of the solution implies a post-reduction 
sample complexity of at least $\exp(k^2\log k)$ in both papers. Our results improve the post-reduction sample complexity 
to $\exp(k\log k)$ and the post-reduction runtime to $O(k^{3+o(1)})$. A more detailed comparison with previous work on 
topic models is given in Section~\ref{sec: topic}.

\paragraph{Related work.}
The $k$-coin problem becomes easier when $m$ is superlinear in $k$, and trivial when 
$m$ is $\Omega(k^2\log k)$. 
Therefore, we focus on the smallest $m$ for which the problem is solvable, which 
is $m=2k-1$ if $k$ is assumed, or $m=2k$ if $k$ needs to be verified. As noted
previously, three prior papers gave algorithms with worse performance than ours.
Roughly stating the results (ignoring dependence on $\zeta$ and on $w_{\min}$),
they are as follows. The paper~\cite{RSS14} solved the problem
in sample complexity $s = k^{O(k^2)}$ and post-sampling runtime of
$\poly(k)$. By {\em post-sampling runtime} we mean the time complexity
of the algorithm after the frequencies $h_j, 0 \leq j \leq m$ (= frequency that $j$ of 
the conditionally-iid coins come up ``heads'') have been
collected. That paper also proves a lower bound of $\exp(k)$ on the sample
complexity needed to solve the problem. Subsequently, a different solution
using near optimal sample complexity $s = k^{O(k)}$, but much worse
post-sampling runtime of $k^{O(k^2)}$, was given in~\cite{LRSS15}.
More recently, an algorithm achieving sample complexity $s = k^{O(k)}$
and post-sampling runtime of $\poly(k)$ was analyzed in~\cite{KKMMR18}.
We  note that all of these papers use $m=2k-1$, and hence do not deal with
verifying that the source is a $k$-coin distribution.
\suppress{
Thus, prior to this work, either the sample complexity or the post-sampling
runtime were $m^{\Omega(k^2)}$. Our algorithm guarantees near optimal 
sample complexity, while at the same time it improves substantially the best
previous post-sampling runtime. Moreover, our algorithm is simpler and
practically appealing in comparison with these previous algorithms.
}

In~\cite{RSS14,LRSS15}, the $k$-coin problem arises as the output 
of a reduction from the problem of identifying topic models, introduced 
in~\cite{Hof99,PRTV00}. A (pure) $k$-topic model is simply analogous
to the $k$-coin problem with highly multi-sided coins. There has been ample 
work on learning pure and mixed topic models, under various restrictive 
assumptions on the model, and 
also without restrictions~\cite{AGM12,AFHKL12,RSS14,LRSS15}.
The reductions of~\cite{RSS14,LRSS15} can be used in conjunction with
our algorithm to reduce the sample complexity and post-sampling runtime
required to solve the topic model problem. This is discussed in Section~\ref{sec: topic}.

In~\cite{KKMMR18}, the $k$-coin problem arises as output of a reduction
from the problem of inferring population histories (see the references therein).
Our results improve both the sample size and the post-sampling runtime 
complexity of the solution. We do note
that the $k$-coin algorithm in~\cite{KKMMR18} could have been used in
conjunction with the reductions in~\cite{RSS14,LRSS15} to solve the topic
model problem. The bounds derived this way would be worse than the
bounds we prove in this paper.

We also mention some generalizations of the $k$-coin problem that were
considered in the literature. 
\suppress{
The question of inverting the moment map
of an arbitrary distribution on $[0,1]$ was considered in~\cite{LRSS15}
in the context of identifying mixed topic models. In that case, the accuracy
of the inversion depends on $m$, even given precise statistics. Another way
to view the $k$-coin problem is as a mixture of $k$ power distributions on
$\{0,1\}^m$. Thus, a natural generalization is to explore }
Most obvious is 
mixtures of $k$ product
distributions on $\{0,1\}^m$. That is, the formulation is the same as ours except that $X_1,\ldots,X_m$ are merely required to be independent, but not necessarily iid, conditional on the hidden variable $U$.
This problem has been the focus 
of considerable research in the past two 
decades~\cite{KMRRSS94,FM99,CGG01,CR08,FOS08,CM19}. 
Clearly, in this case a larger $m$ is no longer purely helpful, since the number of degrees of freedom of the problem also goes up with $m$. 
It should be noted, though, that the strongest results in this sequence, \cite{FOS08} and \cite{CM19}, do not address the problem of \textit{identifying} the source model; rather, they \textit{learn} a model which generates similar statistics. On the positive side, this task can sometimes be performed even under conditions where there is not enough information in the statistics for identification (i.e., when there are models with near-enough statistics that are far apart in, say, transportation distance); but on the negative side, since these algorithms (as well as the algorithm in~\cite{LRSS15}) are forced to perform 
 an exhaustive enumeration over a large grid of potential models, their computational efficiency does not much improve even when the statistics are known to sufficiently-good accuracy that only a very small-diameter (in transportation distance) set of models could generate them.

The distinction between the ``identification'' and ``learning''
goals was made already in~\cite{FM99}, who solved the identification 
problem for mixtures of $k=2$ product distributions on $\{0,1\}^m$. Similar 
results for somewhat more general models were achieved at a similar time 
in~\cite{CGG01}. The best result to date~\cite{CM19} learns in time 
$k^{k^3}\cdot m^{O(k^2)}$, improving upon a previous result~\cite{FOS08} 
of $m^{O(k^3)}$. The same paper~\cite{CM19} shows a lower bound of 
$m^{\Omega(\sqrt{k})}$ on the sample complexity of the task.

Beyond mixtures of product distributions, an even more complex but important class of source identification problems arises when the hidden variable (our ``$U$'') may be just one of several such variables, and when a known directed causal structure exists among the observed variables (the ``$X_i$''). This is a very broad field of investigation and we point only to~\cite{Pea09,PetersJS17} for background, and to~\cite{AHK12} for an example of how (with some additional assumptions on the distributions of the $X_i$) certain models can be handled.

\section{Mixture Models and other Definitions}\label{sec: mix}

\begin{definition}[The $k$-coin model]
A $k$-coin model $\mathcal{M} = (\alpha, w)$ is a mixture of $k$ Bernoulli variables with success probabilities $\alpha_1,\dots,\alpha_k$ with non-negative mixing weights $w_1,\dotsc,w_k$, respectively. 	
\end{definition}

\begin{definition}[$m$-snapshots of a $k$-coin model] 
Given a $k$-coin model 
$\mathcal{M} = (\alpha, w)$, an {\em $m$-snapshot} is a sample from the 
mixture of binomial distributions $w_1\Binomial(m,\alpha_1)+\dotsc+w_k\Binomial(m,\alpha_k)$. (The binomial is a sufficient statistic for $m$ rv's $X_1,\ldots,X_m$ because they are iid given the selected coin.)
\end{definition}

For a $k$-coin model, the moments defined in equation \eqref{moments-hausdorff} can be written as follows where $\delta_\alpha$ being the Dirac measure at $\alpha$,
\[ \mathcal{P} = w_1\delta_{\alpha_1} + \dotsb + w_k\delta_{\alpha_k}, \quad\mu_i = \sum_{j=1}^k \alpha_j^i w_j. \]

\begin{definition}[Separation for polynomials and mixtures] 
For a $k$-coin probability model $\mathcal{M} = (\alpha, w)$, define the separation by $\zeta(\mathcal{M}) = \min_{i\neq j} \Abs{\alpha_i - \alpha_j}$. 
 For a degree $k$ polynomial with roots $\beta_1,\dotsc,\beta_k \in \C$, define the root separation by $\min_{i\neq j} \Abs{\beta_i - \beta_j}$. 
\end{definition}

\begin{definition}
The rectangular Vandermonde matrix $V^{(m)}_{\alpha} \in \R^{(m+1)\times k}$ associated with a vector $\alpha \in \C^k$ is given by 
\[ V_{\alpha}^{(m)} = \begin{bmatrix}
 1 & 1 & 1 & \dotsm & 1\\
 \alpha_1 & \alpha_2 & \alpha_3 & \dotsm & \alpha_k\\
 \alpha_1^2 & \alpha_2^2 & \alpha_3^2 & \dotsm & \alpha_k^2\\
 \vdots & \vdots & \vdots & \ddots & \vdots\\
 \alpha_1^{m} & \alpha_2^{m} & \alpha_3^{m} & \dotsm & \alpha_{k}^{m}
 \end{bmatrix} \] We'll denote the square Vandermonde matrix with $V_{\alpha} \coloneqq V_{\alpha}^{(k-1)}$. 
\end{definition}

\begin{definition}[Hankel Matrix] The $(k+1)\times (k+1)$ 
Hankel matrix $\mathcal{H}_{k+1} = \mathcal{H}_{k+1}(\mathcal{P})$ is defined as:
\begin{equation}\label{eq: Hankel def}
\begin{array}{l}
\mathcal{H}_{k+1} = \begin{bmatrix}
                           \mu_0 & \mu_1 & \mu_2 & \cdots & \mu_k \\
                           \mu_1 & \mu_2 & \mu_3 & \cdots & \mu_{k+1} \\
                           \vdots & \vdots & \vdots &             & \vdots \\
                           \mu_k & \mu_{k+1} & \mu_{k+2} & \cdots & \mu_{2k}
                           \end{bmatrix}.
\end{array}
\end{equation}
\end{definition}

Note that if $\mathcal{P}$ is supported on a 
set of cardinality $k$ (a.k.a.\ a $k$-coin
distribution), then
\begin{equation}\label{eq: Hankel}
\mathcal{H}_{k+1} = \sum_{j=1}^k w_j \alpha_j \alpha_j^\tpose = V_\alpha^{(k)} \diag(w_1,\ldots,w_k) {V_\alpha^{(k)}}\tpose
\end{equation}
where $\alpha_j^\tpose = (1, \alpha_j, \alpha_j^2, \alpha_j^3, \dots, \alpha_j^k)$. This also shows that the Hankel matrix is positive semi-definite. 
\begin{definition}[Polynomial associated with a vector]
We associate to each vector $q\in \R^{k}$ a degree $k-1$ polynomial 
$\hat{q}(x) = \sum_{j=0}^{k-1} q_jx^j$. (For this reason we use zero indexing for the vector.)
\end{definition}

\begin{definition}
For a matrix $M$, let $\Norm{M}_2$ denote the $2\to 2$ operator norm of $M$. Thus, $\Norm{M}_2 = \sigma_{\max}(M)$, the largest singular value of $M$. 
\end{definition}

\begin{definition}
For a Hermitian matrix $M$, let $\lambda_i(M)$ denote the $i$th smallest eigenvalue of $M$.
In particular $\lambda_1(M)$ is the smallest eigenvalue of $M$.
\end{definition}

\begin{definition}[Euclidean projection onto a closed convex set]
For a closed convex set $S \subseteq \R^k$ and any point $x \notin S$, the Euclidean projection of $x$ onto $S$ is $ \Proj_S(x) \coloneqq \argmin_{y\in S} \Norm{y-x}_2. $ This projection is unique. 
\end{definition}

\section{Properties of Hankel Matrices}\label{sec: prelim}
We begin with some properties of Hankel matrices corresponding to finitely supported distributions that follow from results in Chihara \cite{chihara78}. (See Schmudgen \cite[Ch. 10]{schmudgen2017moment} for a complete characterization.) For completeness, a proof is provided in the appendix.\begin{lemma}\label{lm: kernel characterization}
Let ${\cal P}$ be a probability measure on $[0,1]$. Then,
\begin{enumerate}
\item $\mathcal{P}$ is supported on a set of cardinality at most $k$ iff $\mathcal{H}_{k+1}$ is singular. 
\item If the support of $\mathcal{P}$ is a set $\Set{\alpha_1,\dotsc,\alpha_k} \subset [0,1]$ then the kernel of $\mathcal{H}_{k+1}$ is spanned by the vector $q\in \R^{k+1}$ where $\hat{q}(z) = \prod_{i=1}^k (z-\alpha_i)$ is the unique monic polynomial with roots at the support of $\mathcal{P}$.
\end{enumerate} \emph{Proof in Appendix A.}
\end{lemma}
We prove a quantitative version of the above lemma.
\begin{lemma}\label{lm: eigenvalue}
Let $\mathcal{P}=(\alpha,w)$ be a $k$-coin distribution with separation $\zeta$, and let $\mathcal{H}_k \coloneqq \mathcal{H}_k(\mathcal{P})$. For every monic degree $k'\leq k-1$ polynomial represented by $q \in \R^k$, \[
q^\tpose \mathcal{H}_k q \ge \frac{w_{\min}}{k}\cdot \left(\frac{\zeta}{16}\right)^{2k-2}\cdot \Norm{q}_2^2.
\]
\end{lemma}
\begin{proof} 

Let $\beta_1, \beta_2, \dots, \beta_{k'}$ be the roots (possibly complex) of the polynomial $\hat{q}$, ordered so that $\Abs{\beta_1}\geq \Abs{\beta_2} \geq \dotsb \geq \Abs{\beta_{k'}}$. Since $\hat{q}$ is monic, we can write $\hat{q}(x) = \prod_{j=1}^{k'}(x-\beta_j)$. 
As the balls
$B(\alpha_i,\zeta/2)$, $i=1,2,\dots,k$, are disjoint, by the pigeonhole 
principle, there exists an $i\in\{1,2,\dots,k\}$
such that $B(\alpha_i,\zeta/2) \cap \{\beta_1,\beta_2,\dots,\beta_{k'}\} = \emptyset$. The value of $\hat{q}$ at $\alpha_i$ is
\[ \hat{q}(\alpha_i) = \prod_{j=1}^{k'} (\alpha_i - \beta_j).\]
There must be some 
 $\ell\in\{0,1,2,\dots,k'\}$ such that $\Abs{q_{\ell}}^2\geq \frac{\Norm{q}_2^2}{k'+1}$. Notice that $\Abs{q_{\ell}}=
\Abs{e_{k'-\ell}(\beta_1,\beta_2,\dots,\beta_{k'})}$, where $e_r$ is the $r$-th elementary
symmetric polynomial over $k'$ variables. ($e_0 = 1,e_1=\sum \beta_i,e_2=\sum_{i<j} \beta_i \beta_j $ etc.) 
So, $e_{k'-\ell}$ is the sum over
${k' \choose k'-\ell}\leq 2^{k'}$ monomials, hence
$\Abs{\beta_1\beta_2\cdots\beta_{k'-\ell}}\geq \frac{\Norm{q}_2}{(\sqrt{k'+1}) 2^{k'}}$. Eliminating from the product all the factors whose absolute value is below $2$, we get that for some $r \leq k'-\ell$,
$\Abs{\beta_1\beta_2\cdots\beta_r}\geq \frac{\Norm{q}_2}{(\sqrt{k'+1}) 4^{k'}}$. 
For $j\in\Set{1,2,\dots,r}$, since $\Abs{\beta_j} \ge 2$ and $\alpha_i\in [0,1]$, it follows that $\Abs{\alpha_i - \beta_j}\geq \frac{\Abs{\beta_j}}{2}$.
Also, by the definition of $i$ we have that $\Abs{\alpha_i - \beta_j} > \zeta/2$ for all $j\in\Set{1,2,\dots,k'}$. Thus, we have that
\begin{align*}
\Abs{\hat{q}(\alpha_i)} 
&= \Paren{\prod_{j=1}^r \Abs{\alpha_i - \beta_j}} \Paren{\prod_{j=r+1}^{k'} \Abs{\alpha_i - \beta_j}} \geq \frac{\Abs{\beta_1\beta_2\cdots\beta_r}}{2^r} \Paren{\frac{\zeta}{2}}^{k'-r} \\
&\geq \frac{\Norm{q}_2}{(\sqrt{k'+1})8^{k'}} \zeta^{k'} 
\geq  \frac{1}{\sqrt{k}}\cdot \Paren{\frac{\zeta}{8}}^{k-1}\Norm{q}_2.
\end{align*}
Therefore,
\[
q^\tpose \mathcal{H}_k q = \sum_{j=1}^k w_j\cdot \Paren{\hat{q}(\alpha_j)}^2 \geq
w_{\min}\cdot\Paren{\hat{q}(\alpha_i)}^2 > 
w_{\min}\cdot \frac{1}{k}\cdot \Paren{\frac{\zeta}{8}}^{2k-2}\cdot\Norm{q}_2^2.
\]
\end{proof}


\begin{corollary}\label{cor: lambda_2} For a $k$-coin model $(\alpha,w)$,
$\lambda_2(\mathcal{H}_{k+1}) > w_{\min}\cdot \Paren{\frac{\zeta}{16}}^{2k-2}$.
\end{corollary}

\begin{proof}
By the Courant-Fischer-Weyl min-max principle, the smallest eigenvalue of $\mathcal{H}_k$ is given
by minimizing the Rayleigh-Ritz quotient. Let $q \neq 0$ be a minimizer of $\frac{q^{\tpose}\mathcal{H}_kq}{q^{\tpose}q}$. 
\suppress{
We show that $q_{k-1}\neq 0$. Assume for contradiction that $q_{k-1} = 0$. Then $\hat{q}(x)$ is a polynomial of degree $k' < k-1$, $\hat{q}(x) = \sum_{j=0}^{k'}q_jx^j$. 
Let $\hat{p}(x) = x^{k-1-k'}\hat{q}(x)$ be the polynomial obtained by adding $k-1-k'$ roots at $0$, i.e., $p_j =  q_{j-k-1-k'}$ for $j\geq k-1-k'$ 
and $p_j = 0$ for $j< k-1-k'$. Now 
\[  \frac{q^{\tpose}\mathcal{H}_kq}{q^{\tpose}q} = \int_{0}^1 \hat{q}^2(\alpha) \diff \mathcal{P}(\alpha) > \int_{0}^1 \alpha^{k-1-k'} \hat{q}^2(\alpha)\diff \mathcal{P}(\alpha) = \int_{0}^1 \hat{p}^2(\alpha) \diff \mathcal{P}(\alpha) = \frac{p^{\tpose}\mathcal{H}_kp}{p^{\tpose}p}, \] which is a contradiction. 

\ljs{After strengthening previous lemma, can start proof from here.}
Now we know that $q_{k-1}\neq 0$ for any minimizer of the Rayleigh quotient so we can scale the minimizer to have $q_{k-1}=1$, so
\[
\lambda_1(\mathcal{H}_k) = \min_{q\neq 0}\frac{q^\tpose \mathcal{H}_k q}{q^\tpose q} = \min_{\substack{q\neq 0\\q_{k-1}=1}}\frac{q^\tpose \mathcal{H}_k q}{q^\tpose q}\geq 
w_{\min}\cdot \Paren{\frac{\zeta}{32}}^{2k-2},
\]
where the last inequality follows from Lemma~\ref{lm: eigenvalue}.
Notice that $\mathcal{H}_k$ is the orthogonal projection of $\mathcal{H}_{k+1}$ onto the first $k$ coordinates.
Therefore, by the Cauchy interlacing theorem~\ref{thm: Cauchy}, $\lambda_2(\mathcal{H}_{k+1})\ge \lambda_1(\mathcal{H}_k)$.
}
Let $k'$ be greatest such that $q_{k'}\neq 0$, and w.l.o.g.\ set $q_{k'}=1$. Then by Lemma~\ref{lm: eigenvalue}, 
\[
\lambda_1(\mathcal{H}_k) 
= \min_{q\neq 0}\frac{q^\tpose \mathcal{H}_k q}{q^\tpose q} 
\geq \frac{w_{\min}}{k}\cdot \Paren{\frac{\zeta}{8}}^{2k-2} \geq w_{\min}\cdot \Paren{\frac{\zeta}{16}}^{2k-2},
\] where the last inequality follows from observing that $1/k \geq 1/2^{2k-1}$ for $k\geq 2$. 
Notice that $\mathcal{H}_k$ is a principal submatrix of $\mathcal{H}_{k+1}$.
Therefore, by the Cauchy interlacing theorem (Theorem~\ref{thm: Cauchy}), $\lambda_2(\mathcal{H}_{k+1})\ge \lambda_1(\mathcal{H}_k)$.
\end{proof}


\section{The Empirical Moments}\label{sec: empirical}

We bound the sampling error as follows. Sample $s$ coins and let each of the random variables 
$h_j$, $0\leq j \leq 2k$, be the fraction of coins which came up ``heads'' exactly $j$ times.
Then by the additive deviation bound known as Hoeffding's inequality~\cite{Hoeff}, $\Pr(|h_j-E(h_j)|\geq t ) \leq 2\exp(-2t^2 s)$.
Thus 
\begin{lemma} If we use $s>\frac{1}{2t^2} \log (4k/\delta)$ samples then with probability at least $1-\delta$:  $\forall j$, $|h_j-E(h_j)|<t$. \label{lem:better concentration of histogram} \end{lemma}

We can convert between the normalized histogram $h$ and the standard moments of the 
distribution by using the observation (Lemma~1 in \cite{Rivlin1970BoundsOA}) that for any 
$t\in \R$, 
\[ t^i = \sum_{j=i}^n \frac{\binom{j}{i}}{\binom{n}{i}} \times \binom{n}{j} t^j(1-t)^{n-j} \]
This gives us a linear transformation for converting from $h$ to the vector 
$\tilde{\mu}= (\tilde{\mu}_0,\dotsc,\tilde{\mu}_{2k})$. 
Define $\Pas \in \R^{(2k+1)\times (2k+1)}$ (using zero-indexing) by 
\[ \Pas_{ij} = 
\begin{cases} 
\frac{\binom{j}{i}}{\binom{2k}{i}} &\quad\text{if $j \geq i$}\\
0&\quad\text{otherwise;}	
\end{cases} \] 
then $\tilde{\mu} = \Pas h$.

\begin{lemma}\label{lem:pascal operator norm}
$\Norm{\Pas}_2 \leq 6^k$. {\em Proof in Appendix~\ref{tossed}.}
\end{lemma}

Now let $\mu = (\mu_0,\dotsc,\mu_{2k})$ be the actual vector of moments of the distribution $\mathcal{P}$.

\begin{lemma}\label{lm: sample quality}
For every $\eps > 0$, using $s = 2^{O(k)}\cdot\frac{1}{\eps^2}\cdot\log(1/\delta)$ samples gives us 
estimated moments $\tilde{\mu} = (\tilde{\mu}_0,\dotsc,\tilde{\mu}_{2k})$ satisfying $\Norm{\tilde{\mu}-\mu}_{\infty} \leq \eps$
with probability at least $1-\delta$.
\end{lemma}

\begin{proof}
Follows directly from Lemma~\ref{lem:pascal operator norm} and Lemma~\ref{lem:better concentration of histogram}.
\end{proof}

Given an $s$-sample as above with empirical moments 
$\tilde{\mu}_0,\tilde{\mu}_1,\dots,\tilde{\mu}_{2k}$,
denote by $\tildecal{H}_{k+1}$ the empirical Hankel matrix
\begin{equation}\label{eq: empirical Hankel}
\tildecal{H}_{k+1} = \begin{bmatrix}
\tilde\mu_0 & \tilde\mu_1 & \tilde\mu_2 & \cdots & \tilde\mu_k\\
\tilde\mu_1 & \tilde\mu_2 & \tilde\mu_3 & \cdots & \tilde\mu_{k+1}\\
\vdots & \vdots & \vdots & \ddots & \vdots \\
\tilde\mu_k & \tilde\mu_{k+1} & \tilde\mu_{k+2} & \cdots & \tilde\mu_{2k} \end{bmatrix}
\end{equation}

\begin{corollary}\label{cor: Hankel quality}
For every $\eps > 0$, using $s = 2^{O(k)}\cdot\frac{1}{\eps^2}\cdot\log(1/\delta)$ samples, 
we can obtain an empirical Hankel matrix satisfying $
\Norm{\tildecal{H}_{k+1}-\mathcal{H}_{k+1}}_{2} \leq \eps $
with probability at least $1-\delta$. 	
\end{corollary}

\begin{proof}
We have $\Norm{\tildecal{H}_{k+1}-\mathcal{H}_{k+1}}_{2} \leq \Norm{\tildecal{H}_{k+1}-\mathcal{H}_{k+1}}_{\Frob} \leq 
(k+1)\cdot \|\tilde{\mu} - \mu\|_{\infty}$.
Now use Lemma~\ref{lm: sample quality} with $\frac{\eps}{k+1}$.
\end{proof}

\section{Learning the Source}\label{sec: alg}

In this section, we define our learning algorithm, and we state and prove our main result 
and applications. The auxiliary lemmas are stated and proved in Section~\ref{sec: analysis}.
The algorithm is specified given $k$, lower bounds on the source parameters $\zeta$
and $w_{\min}$, the empirical histogram $h$, and a parameter $\gamma$ controlling the 
output accuracy. 
See Algorithm~\ref{fig:algo} for the full description of the algorithm (where the parameter for probability of success, $1-\delta$, has been suppressed in favor of a constant ``$0.99$'').
\begin{algorithm}
\begin{algorithmic}[1]
\Procedure{LearnCoinMixture}{$k,\zeta,w_{\min},h,\gamma$} 
\State $\tilde\mu \gets \Pas h$
\State $\tildecal{H}_{k+1} \gets \hbox{Hankel}(\tilde\mu)$
\State $v \xleftarrow{\eps_1-\operatorname{approx}} \argmin\{v^{\tpose} \tildecal{H}_{k+1} v:\ v^\tpose v = 1\}$ 
           \Comment{$\eps_1 = w_{\min}\cdot 2^{-\gamma}\cdot (\zeta/16)^{2k}$}
\State $\tilde\beta_1,\tilde\beta_2,\dotsc,\tilde\beta_k \xleftarrow{\eps_2-\operatorname{approx}} \roots(\hat{v})$ 
           \Comment{$\eps_2 = \frac{1}{6k}\cdot 2^{-\gamma}\cdot (\zeta/2)^k$}
\State $\tilde{\alpha}_1,\tilde{\alpha}_2,\dots,\tilde{\alpha}_k \gets \Proj_{[0,1]}(\tilde\beta_1),\dots,\Proj_{[0,1]}(\tilde\beta_k)$ 
\State $\tilde{w} \gets \textsc{RectifyWeights}(V_{\tilde{\alpha}}^{-1}\tilde{\mu})$  \label{op: vandermonde}
           \Comment{see Algorithm~\ref{fig:algo2} on page~\pageref{fig:algo2}}
\State Output $\tilde{\cal M} = (\tilde{\alpha},\tilde{w})$ 
\EndProcedure	
\end{algorithmic} \caption{Algorithm \textsc{LearnCoinMixture}} \label{fig:algo}
\end{algorithm}

\begin{theorem}\label{thm: main}
Let $\mathcal{M}=(\alpha, w)$ be a $k$-coin model with separation $\zeta = \zeta(\mathcal{M})$. 
For any $\gamma \geq 1$, Algorithm~\ref{fig:algo} uses a histogram $h$ for a sample of $2k$-snapshots
of size $s = w_{\min}^{-2}\cdot 2^{O(k+\gamma)}\cdot \zeta^{-O(k)}\cdot \log\delta^{-1}$, 
and outputs a model $\tildecal{M} = (\tilde{\alpha}, \tilde{w})$ satisfying 
\begin{align*}
\Norm{\alpha-\tilde{\alpha}}_{\infty},\Norm{w-\tilde{w}}_{\infty} \leq 2^{-\gamma}
\end{align*}
with probability at least $1 - \delta$.
After sampling, Algorithm~\ref{fig:algo}
computes the approximate model $\tildecal{M}$ using 
$O(k^2\log k + k\log^2 k\cdot 
\log(\log\zeta^{-1} + \log w_{\min}^{-1} + \gamma) 
)$
arithmetic operations.
\end{theorem}

\begin{proof}
Throughout the proof, we make no attempt to optimize the absolute constants that are used.
Let $u_1$ denote the unit vector spanning the kernel of ${\cal H}_{k+1}$, and let $v_1$ denote the eigenvector
corresponding to the smallest eigenvalue of $\tilde{\cal H}_{k+1}$. Also, let $\eps_0 > 0$ be a sufficiently small
constant, to be determined later.
The analysis of Algorithm~\ref{fig:algo} can be broken down into steps, each of which degrades 
the accuracy obtained in the initial sampling. 
The outline is as follows. The auxiliary claims and proofs appear mostly in Section~\ref{sec: analysis}.
\begin{enumerate}
\item We assume that 
         $\Norm{\tildecal{H}_{k+1}-\mathcal{H}_{k+1}}_2 \leq w_{\min}\cdot 2^{-\gamma}\cdot (\zeta/16)^{4k}$.
         This is guaranteed by Lemma~\ref{lm: sample quality} and Corollary~\ref{cor: Hankel quality}
         for a sample of size $s = w_{\min}^{-2}\cdot 2^{O(k+\gamma)}\cdot\zeta^{-O(k)}\cdot\log\delta^{-1}$,
         with probability at least $1-\delta$.
         
\item As $\Norm{\tildecal{H}_{k+1}-\mathcal{H}_{k+1}}_2 \leq w_{\min}\cdot 2^{-\gamma}\cdot (\zeta/16)^{4k}$,
         by Lemma~\ref{lm: kernel stability},
         $$
         \Norm{u_1 - v_1}_2 < \sqrt{2(k+1)}\cdot 2^{-\gamma}\cdot (\zeta/16)^{2k} < \frac 1 2\cdot 2^{-\gamma}\cdot (\zeta/8)^{2k}.
         $$

\item We use Lemma~\ref{lm: PC} with $\eps = w_{\min}\cdot 2^{-\gamma}\cdot (\zeta/16)^{2k}$, which satisfies
         the conditions of the lemma. We compute $v\in \R^{k+1}$ such that 
         $$
         \Norm{v-v_1}_2 \leq \eps < \frac 1 2 \cdot2^{-\gamma}\cdot (\zeta/16)^{2k},
         $$
         using 
         $O(k^2\log k + k\log^2 k\cdot \log(\log\zeta^{-1} + \log w_{\min}^{-1} + \gamma))
         $ 
         arithmetic operations.

\item As $\|u_1 - v_1\|_2, \|v - v_1\|_2 < \frac 1 2\cdot 2^{-\gamma}\cdot (\zeta/16)^{2k}$, we have that
         $\|u_1 - v\|_2 < 2^{-\gamma}\cdot (\zeta/16)^{2k}$. So, by Lemma~\ref{lm: empirical kernel},
         $$
         \Norm{q-r}_{\infty} < 2^k\cdot \sqrt{k+1}\cdot 2^{-\gamma}\cdot (\zeta/16)^{2k} < 2^{-\gamma}\cdot (\zeta/8)^{2k},
         $$
         where $q\coloneqq u_1/\Abs{(u_1)_k}, r\coloneqq v/\Abs{(u_1)_k}$.

\item As $\Norm{q-r}_{\infty} < 2^{-\gamma}\cdot (\zeta/8)^{2k}\le \frac{1}{24k(k+1)}\cdot 2^{-\gamma}\cdot (\zeta/2)^{2k-1}$, 
         by Lemma~\ref{lm: root stability} we have that 
         $$
         d(\alpha, \beta) \leq \frac{1}{6(k+1)}\cdot 2^{-\gamma}\cdot (\zeta/2)^k
         $$ 
         (where $\alpha$ is the vector of roots of $\hat{q}$ and $\beta$ is the vector of roots of $\hat{r}$ and $d$ is the matching distance, defined in Lemma~\ref{lm: root stability}).

\item We use Corollary~\ref{cor: rootfinding} with $\rho = \eps = \frac{1}{6(k+1)}\cdot 2^{-\gamma}\cdot (\zeta/2)^k$,
         which satisfy the conditions of the corollary. Thus, we can compute biases 
         $\tilde{\alpha}_1,\dotsc,\tilde{\alpha}_k$ satisfying
         $$
         \Norm{\tilde{\alpha} - \alpha}_{\infty} \leq \rho + \eps \le \frac{1}{3(k+1)}\cdot 2^{-\gamma}\cdot (\zeta/2)^k, 
         $$
         using $O(k \log^2 k\cdot (\log(\log\zeta^{-1} + \gamma) + \log^2 k))$
         arithmetic operations.

\item Finally, line~\ref{op: vandermonde} can be executed in the time its takes to invert the 
        Vandermonde matrix $V_{\tilde{\alpha}}$ (i.e., $O(k^2)$ arithmetic operations, for instance 
        using Parker's algorithm~\cite{Par64}; by Lemma~\ref{lm: rounded weight reconstruction} 
        the procedure $\textsc{RectifyWeights}$ takes $O(k)$ operations). By 
        Corollary~\ref{cor: full weight reconstruction}, as
        $\|\tilde\alpha - \alpha\|_{\infty}, \|\tilde\mu - \mu\|_{\infty} \le \frac{1}{3(k+1)}\cdot 2^{-\gamma}\cdot (\zeta/2)^k$
        (the guarantee for $\tilde{\mu}$ is implied with plenty of room to spare by our assumption on the sample),
        we have $\Norm{\tilde{w}-w}_{\infty} \le 2^{-\gamma}$. \qedhere
\end{enumerate}
\end{proof}

Notice that the proof actually gives a stronger guarantee for $\Norm{\tilde{\alpha} - \alpha}_{\infty}$, which
is relative to $(\zeta/2)^k$. We can get a relative guarantee $\Norm{\tilde{w}-w}_{\infty}\le w_{\min}\cdot 2^{-\gamma}$ 
by increasing the sample size by a factor of $w_{\min}^{-2}$. 
\begin{corollary}\label{cor: Wasserstein}
Let $W(\mathcal{M}, \tilde{\mathcal{M}})$ denote the Wasserstein distance between models 
$\mathcal{M}$ and $\tilde{\mathcal{M}}$ (viewed as metric measure spaces on $[0,1]$).
Then, $W(\mathcal{M}, \tilde{\mathcal{M}}) \leq (k+1)\cdot 2^{-\gamma}$ with probability at 
least $0.99$. 
\end{corollary}
\begin{proof}
Each $\alpha_i$ can be matched to its corresponding $\tilde{\alpha}_i$ up to weight 
$\min\{w_i, \tilde{w_i}\}$. The additional $\abs{w_i-\tilde{w_i}}$ must move an additional 
distance of at most $1$. This gives
\begin{align*}
W(\mathcal{M}, \tilde{\mathcal{M}}) &\leq \sum_{i=1}^k \abs{\alpha_i - \tilde{\alpha}_i}\cdot \min\{w_i, \tilde{w}_i\} + \sum_{i=1}^k \abs{w_i - \tilde{w}_i}\\
&\leq \sum_{i=1}^k 2^{-\gamma}\min\{w_i, \tilde{w}_i\} + \sum_{i=1}^k 2^{-\gamma}\\
&\leq (k+1)\cdot 2^{-\gamma},
\end{align*}
using Theorem~\ref{thm: main} and the fact that $\sum_{i=1}^k \min\{w_i, \tilde{w}_i\}\le\sum_{i=1}^k w_i = 1$.
\end{proof}

\section{Implications for Topic Models}\label{sec: topic}

Theorem~\ref{thm: main} improves upon the upper bound of Theorem 5.1 in~\cite{RSS14}, which uses a 
sample of $(2k-1)$-snapshots of size $\max\left\{(2/\zeta)^{O(k)}, \left(2^{\gamma} k\right)^{O(k^2)}\right\}$
to achieve accuracy $2^{-\gamma}$ with high probability, using runtime of $O(k^{c})$ arithmetic
operations, for a relatively large constant $c$ (in particular, the algorithm solves a convex quadratic
program whose representation uses $k^3$ bits).
Theorem~\ref{thm: main} also improves upon the upper bound 
of~\cite{LRSS15}.\footnote{See Theorem 3.9 in the ArXiv version: {\tt https://arxiv.org/pdf/1504.02526.pdf}}
That algorithm uses a sample size comparable to ours, but requires runtime $\left(2^{\gamma} k\right)^{O(k^2)}$
to achieve accuracy $2^{-\gamma}$ with high probability.

These improvements imply immediately a similar 
improvement for learning pure $k$-topic models, using known reductions from $k$-topic models to $k$-coin
models. The reduction in Theorem 4.1 of~\cite{RSS14} uses a sample of $1$- and $2$-snapshots of size
$O\left(n\cdot \poly\left(\log n, k, w_{\min}^{-1}, \zeta^{-1}, 2^\gamma\right)\right)$, and runtime polynomial
in the sample size, to reduce the problem to solving $k$ instances of the $k$-coin problem with accuracy 
$\min\left\{\left(2^{\gamma} k/w_{\min}\zeta\right)^{-O(1)}, \left(2^{\gamma} k\right)^{-O(k)}\right\}$.
The reduction in~\cite{LRSS15} \footnote{See Theorem 6.1 in the ArXiv version.} uses a sample of
$1$- and $2$-snapshots of size $\poly\left(n,k,2^{\gamma}\right)$, and runtime polynomial in the sample
size, to reduce the problem to solving at most $k$ instances of the $k$-coin problem with accuracy 
$\left(2^{\gamma} k\right)^{-O(k)}$. Notice that solving the $k$-coin outcome of either one of the two
reductions using either one of the two previous algorithms requires a sample size of at least $k^{O(k^2)}$
(on account of the required accuracy). Our algorithm enables a solution to the outcome of these reductions
using a sample size of $k^{O(k)}$ (and total runtime of $O(k^{3+o(1)})$). We note that the accuracy
in~\cite{RSS14,LRSS15} is stated in terms of Wasserstein distance, which is a weaker guarantee than
the one we use here (see Corollary~\ref{cor: Wasserstein}).

\section{Analysis}\label{sec: analysis}

In this section we prove the lemmas that are needed in the proof of Theorem~\ref{thm: main}.
We have to cope with the fact that roots of polynomials (and even, generally, of polynomials with well-separated roots), are notoriously ill-conditioned in terms of the polynomial coefficients~\cite{Wilkinson84}. For this reason we will be developing bounds specifically adapated to our situation. We begin with an estimate on the accuracy of the recovered kernel of the Hankel matrix.

%
%

\subsection{Approximating the kernel of {${\mathbf {\cal H}_{k+1}}$}}

\begin{lemma}
\label{lm: kernel stability}
Let $\mathcal{P}$ be any $k$-coin distribution with separation $\zeta$. Then, for every 
$\eps < w_{\min}\cdot\left(\frac{\zeta}{16}\right)^{2k}$
the following holds. Suppose that $\Norm{\tildecal{H}_{k+1}-\mathcal{H}_{k+1}}_2 \leq \eps$.
Let $u_1$ be the unit vector in the kernel of $\mathcal{H}_{k+1}$ and let $v_1$ be the unit 
eigenvector corresponding to $\lambda_1(\tildecal{H}_{k+1})$ (chosen so that $u_1^{\tpose}v_1 \geq 0$). 
Then $\Norm{u_1-v_1}_2 < \sqrt{2(k+1)}\cdot\Paren{\frac{16}{\zeta}}^{2k}\cdot \frac{\eps}{w_{\min}}$.
\end{lemma}

\begin{proof}
By Weyl's inequality, we have that $\lambda_1(\tildecal{H}_{k+1}) \leq \eps$.  
By Corollary~\ref{cor: lambda_2}, the eigengap $\lambda_2(\mathcal{H}_{k+1}) - \lambda_1(\tildecal{H}_{k+1})$ is at least 
\[ w_{\min}\Paren{\frac{\zeta}{16}}^{2k-2} - w_{\min}\Paren{\frac{\zeta}{16}}^{2k} > w_{\min} \Paren{\frac{\zeta}{16}}^{2k}. \] 
Now we can use Corollary~\ref{cor: DK 1-dim} to obtain 
\begin{align*}
u_1^{\tpose}v_1 
&= \Abs{u_1^{\tpose}v_1} \geq \Paren{1-\frac{\Norm{\mathcal{H}_{k+1}-\tildecal{H}_{k+1}}_{F}^2}{\Abs{\lambda_2(\mathcal{H}_{k+1})-\lambda_1(\tildecal{H}_{k+1})}^2}}^{1/2} > 1-\frac{(k+1)\cdot \eps^2}{w^2_{\min}\Paren{\frac{\zeta}{16}}^{4k}}\\
&= 1-(k+1)\cdot\Paren{\frac{16}{\zeta}}^{4k}\cdot \left(\frac{\eps}{w_{\min}}\right)^2.
\end{align*}
Since $\Norm{u_1-v_1}_2^2 = 2 - 2 u_1^{\tpose}v_1$ we get that
\[ \Norm{u_1 - v_1}_2^2 < 2(k+1)\cdot\Paren{\frac{16}{\zeta}}^{4k}\cdot \left(\frac{\eps}{w_{\min}}\right)^2.\qedhere \]
\end{proof}

Recall that $\left(\lambda_1(\tildecal{H}_{k+1}),v_1\right)$ is an eigenpair of $\tildecal{H}_{k+1}$. We
need to compute a good approximation of $v_1$. This can be done using the following lemma. The
result is implied by the algorithm of Pan and Chen (Theorem 1.2 of~\cite{PC99}). Extracting our lemma from the result in that paper is somewhat involved and we provide in Appendix~\ref{tossed} a brief outline of the argument
(in particular, the parts that are not spelled out in that paper).
\begin{lemma}\label{lm: PC}
For every $\eps$ such that $0 < \eps \ll \min\{\lambda_2(\tildecal{H}_{k+1})-\lambda_1(\tildecal{H}_{k+1}),1\}$, 
we can compute a unit vector $v$ satisfying $\Norm{v-v_1}_2 < \eps$ using 
$O\left(k^2\log k + k\log^2 k\log\log(1/\eps) \right)$ arithmetic operations. 
\end{lemma}
\begin{proofsketch}
We follow the outline in the papers by Pan, Chen, and Zheng~\cite{PC99,PCZ98}. As $\tildecal{H}_{k+1}$
is a Hankel matrix, a similarity transformation $A = T\ \tildecal{H}_{k+1}\ T^{-1}$, where
$A$ is tridiagonal, can be computed in time $O(k^2\log k)$. The characteristic polynomial $c_A(x)$ 
of $A$ can then be computed in time $O(k)$. Then, a root $\tilde{\lambda}$ that satisfies 
$|\tilde{\lambda} - \lambda_1(\tildecal{H}_{k+1})| < \eps^2$ can be computed in time
$O\left((k\log^2 k)(\log\log(1/\eps) + \log^2 k)\right)$ (see Theorem~\ref{thm: Pan}; note that
$\|A\|_2 = \|\tildecal{H}_{k+1}\|_2$, thus it is trivially upper bounded by $(k+1)^2$). Next, proceed to
compute $v$ as follows. Pick an initial guess $v^{(0)}$ uniformly at random on the unit sphere (i.e.,
from the unit Haar measure on the sphere). We need $v_1^\tpose v^{(0)} > \frac{1}{\sqrt{k}}$, 
which happens with constant probability. To boost the success probability to $1-\delta$, we can repeat 
the entire process $O(\log(1/\delta))$ times. For constant $\delta$, this does not affect the
asymptotic bound. We compute $v^{(1)},v^{(2)},\dots$ using the
inverse power iteration (see, for instance, Chapter 4 in~\cite{Par98}): Solve for $\tilde{v}^{(t)}$
the system of linear equations $\left(\tilde{\lambda} I - \tildecal{H}_{k+1}\right) \tilde{v}^{(t)} = v^{(t-1)}$,
then set $v^{(t)} = \frac{\tilde{v}^{(t)}}{\|\tilde{v}^{(t)}\|_2}$. As $\tildecal{H}_{k+1}$ is a Hankel
matrix, this can be done using $O(k^2)$ arithmetic operations. 
How many iterations are needed?---It is known that
if $\lambda_1(\tildecal{H}_{k+1})$ is the unique eigenvalue of $\tildecal{H}_{k+1}$ that is closest
to $\tilde{\lambda}$, and if $v_1^\tpose v^{(0)} > 0$, then $\tan \theta^{(t)}\le\rho\cdot\tan\theta^{(t-1)}$, 
where $\theta^{(t)}$ is the angle between $v_1$ and $v^{(t)}$, and 
$\rho = \frac{|\tilde{\lambda} - \lambda_1(\tildecal{H}_{k+1})|}{|\tilde{\lambda} - \lambda_2|}$, where
$\lambda_2$ is an eigenvalue of $\tildecal{H}_{k+1}$ that is second-closest to $\tilde{\lambda}$.
Notice that in our case
$\rho = \frac{|\tilde{\lambda} - \lambda_1(\tildecal{H}_{k+1})|}{\min_{i > 1} |\tilde{\lambda} - \lambda_i(\tildecal{H}_{k+1})|} 
< \frac{\eps^2}{\eps - \eps^2} < 2\eps$. As $\tan \theta^{(0)}\le \sqrt{k}$, after $t = O(\log_{1/2\eps} k)$
iterations, we have $\tan \theta^{(t)} < \eps$. This implies that $\Norm{v^{(t)}-v_1}_2 < \eps$.
\end{proofsketch}

\subsection{The roots of the approximate kernel polynomial}


We need to show that our computed eigenvector of the empirical Hankel matrix is close to
the true eigenvector of the true Hankel matrix.
\begin{lemma}\label{lm: empirical kernel}
Let $\mathcal{P}$ be any $k$-coin distribution with separation $\zeta$. Let $u_1$ be a unit vector in 
the kernel of $\mathcal{H}_{k+1}$. Let $v$ be a unit vector satisfying $\Norm{u_1-v}_2 < \eps$ for 
some $\eps > 0$. Let $q = u_1/\Abs{(u_1)_k}$ and let $r = v/\Abs{(u_1)_k}$. Then 
$\Norm{q-r}_{\infty} < 2^k\sqrt{k+1}\cdot \eps$.
\end{lemma}

\begin{proof}
Notice that $q$ and $r$ are well-defined, as $(u_1)_k\ne 0$ by the second part of Lemma~\ref{lm: kernel characterization}. 
Now each of the coefficients of $q$ can be bounded by 
\[ \Abs{q_i} = \Abs{e_{k-i}(\alpha_1,\dotsc,\alpha_k)} \leq \binom{k}{i} \] where $e_r$ is the $r$-th elementary symmetric 
polynomial over $k$ variables. Now $\Norm{q}_{2} \leq \sqrt{k+1}\cdot \Norm{q}_1 \leq 2^k\sqrt{k+1}$. 
Since $\Abs{(u_1)_k}\cdot \Norm{q}_2 = \Norm{u_1}_2 = 1$, we have $\Abs{(u_1)_k} \leq \frac{1}{2^k\sqrt{k+1}}$, and
\[ \Norm{q-r}_{\infty} \leq \Norm{q-r}_2 \leq 2^k\sqrt{k+1}\cdot \Norm{u_1-v}_2 < 2^k\sqrt{k+1}\cdot \eps, \]
as stipulated.
\end{proof}

We're going to use the roots of the polynomial $\hat{r}$ as our guessed coin biases (after projecting the roots back to $[0,1]$). We first need to show that the roots of $\hat{q}$ are well-behaved with respect to perturbations of $q$ so that when $q$ and $r$ are close the roots of $\hat{q}$ are close to the roots of $\hat{r}$.
\begin{lemma}
\label{lm: root stability}
Let $q\in \R^{k+1}$ be the vector representing a degree-$k$ monic polynomial with roots 
$\alpha_1,\alpha_2,\dots,\alpha_k$ contained in $[0,1]$. Let $\zeta$ be the root separation for $\hat{q}$. 
Let $r\in \R^{k+1}$ represent another degree-$k$ polynomial. Let $\eps \in \left(0, \frac{(\zeta/2)^k}{4k}\right)$.
If $r$ satisfies $\Norm{q - r}_\infty \leq \eps$, 
then the (possibly complex) roots $\beta_1, \beta_2,\dots,\beta_k$ of $\hat{r}$ 
satisfy
\[
d(\alpha,\beta) \leq \frac{4k\eps}{(\zeta/2)^{k-1}}
\] where $d(\alpha,\beta)$ is the optimal matching distance defined by 
\[ d(\alpha, \beta) \coloneqq \min_{\sigma \in \mathbb{S}_k}\max_{i} \Abs{\alpha_i - \beta_{\sigma(i)}}. \]
\end{lemma}

\begin{proof}
Fix any root $\alpha_i$ of $\hat{q}$, and consider the ball 
$$
B_i = B\left(\alpha_i, \frac{4k\eps}{(\zeta/2)^{k-1}}\right)
$$
in the complex plane. By assumption, $\frac{4k\eps}{(\zeta/2)^{k-1}} < \frac{\zeta}{2}$, 	
so there are no other roots of $\hat{q}$, aside from $\alpha_i$, in $B_i$. 
Moreover, for any $x\in B_i$, and for any $j\neq i$, we have that $\Abs{x-\alpha_{j}} \geq \frac{\zeta}{2}$. 
Thus for every $x \in \partial B_i$, we have 
\[ \Abs{\hat{q}(x)} = \Abs{(x-\alpha_i)\prod_{j\neq i} (x-\alpha_j)} >\frac{4k\eps}{(\zeta/2)^{k-1}}\Paren{\frac{\zeta}{2}}^{k-1} = 4k\eps.\] 
On the other hand, we also have that $B_i \subset B(0, (2k-1)/(2k-2))$,
as $\alpha_1,\dotsc,\alpha_k \in [0,1]$ and $\zeta \leq \frac{1}{k-1}$. Therefore, $\Abs{x} \leq \frac{2k-1}{2k-2}$,
and thus 
\begin{align*}	
\Abs{\hat{q}(x) - \hat{r}(x)} 
&= \Abs{\sum_{j=0}^k(q_j-r_j)x^j}\\
&\leq \sum_{j=0}^k\Abs{q_j-r_j}\cdot \Abs{x}^j\\
&\leq (k+1)\cdot\Paren{\frac{2k-1}{2k-2}}^k\cdot \Norm{q-r}_\infty\\
&\leq 4k\eps.
\end{align*} 
By Rouch\'e's theorem (Theorem~\ref{thm: Rouche}), we conclude that there is exactly one zero of $\hat{r}$ in $B_i$ and the 
matching distance bound follows immediately.
\end{proof}

Our reconstructed coin biases will be denoted 
$\tilde{\alpha}_1,\tilde{\alpha}_2,\dots,\tilde{\alpha}_k$. We compute these biases by finding the 
roots of $\hat{v}$ (approximately), and then by projecting these roots onto the unit interval. To find 
the approximate roots we can use the following result of Pan.
\begin{theorem}[Pan's Algorithm: Theorem 1.1 of~\cite{pan1996optimal}] \label{thm: Pan}
Given a monic degree $k$ polynomial $\hat{p}$ with roots $\rho_1,\dotsc,\rho_k \in B(0, 1)$
and an accuracy parameter $\gamma > 1$, 
we can compute approximate roots $\tilde{\rho}_1,\dotsc,\tilde{\rho}_k$ satisfying 
$\|\rho - \tilde{\rho}\|_{\infty} \leq 2^{-\gamma}$ in time $O(k\log^2 k \cdot (\log\gamma + \log^2 k))$.
\end{theorem}


\begin{corollary} \label{cor: rootfinding}
Let $q\in \R^{k+1}$ represent the polynomial $\hat{q}(z) = \prod_{i=1}^k (z-\alpha_i)$ where 
$\alpha_1,\dotsc,\alpha_k \in [0,1]$ are $\zeta$-separated, and let $r \in \R^{k+1}$ represent 
a polynomial of degree $k$ with roots $\beta_1,\dotsc,\beta_k$ satisfying $d(\alpha,\beta) \leq \rho < \zeta/2$. 
For every $\eps\in (0, \zeta/2 - \rho)$, we can reconstruct biases $\tilde{\alpha}_1,\dotsc,\tilde{\alpha}_k$ 
satisfying $\Norm{\tilde{\alpha}-\alpha}_{\infty} \leq \rho + \eps$ using
$O(k\log^2 k\cdot(\log\log(1/\eps) + \log^2 k))$ arithmetic operations.
\end{corollary}

\begin{proof}
We'll first find approximate the roots $\tilde{\beta}_1,\dotsc,\tilde{\beta}_k$ of $\hat{r}$ using 
Theorem~\ref{thm: Pan}. Since the roots of $\hat{r}$ are in $B\left(0,\frac{2k-1}{2k-2}\right)$ instead 
of $B(0,1)$, we'll actually find the roots of $\hat{r}\left(\frac{2k-2}{2k-1}z\right)$ and then multiply
by $\frac{2k-1}{2k-2}$ to get the roots of $\hat{r}$ up to accuracy $\eps$ in time 
$O(k\log^2 k\cdot(\log\log(1/\eps) + \log^2 k))$. (Notice that in order to get the desired accuracy 
we need to run Pan's algorithm to get the rescaled roots to within distance $\frac{2k-2}{2k-1}\cdot\eps$; 
this doesn't matter for the purposes of runtime.)

Our output is $\tilde{\alpha}_i \coloneqq \Proj_{[0,1]}(\tilde{\beta}_i)$ for $i=1,\dotsc,k$, where 
we label the roots $\tilde{\beta}_1,\dotsc,\tilde{\beta}_k$ by the permutation achieving the matching 
distance, i.e., the ordering of coordinates so that $\Norm{\alpha - \beta}_{\infty} = d(\alpha, \beta)$. 
Now
\begin{align*} 
\Abs{\alpha_i - \tilde{\alpha}_i} 
&\leq \Abs{\alpha_i - \Re(\tilde{\beta}_i)} \\
&\leq \Abs{\alpha_i - \Re(\beta_i)} + \Abs{\Re(\beta_i) - \Re(\tilde{\beta}_i)} \\
&\leq \Abs{\alpha_i - \beta_i} + \Abs{\beta_i - \tilde{\beta}_i}\\
&\leq \rho + \eps. \qedhere
\end{align*}
\end{proof}

\subsection{Recovering the mixture weights from the roots}

Once we've recovered the parameters $\tilde{\alpha}_1,\dotsc,\tilde{\alpha}_k$, we need to use those to recover mixture weights. This sequence of steps---first solving (approximately) for the roots, then for the mixture weights---is the essence of Prony's method~\cite{Prony1795},~\cite{Hildebrand74} \S 9.4,~\cite{KumaresanTS84}. 
In this section, we'll show that this recovery can be done by solving a linear system without paying too great a price in terms of accuracy.

We'll begin by stating results characterizing the condition number of a Vandermonde system under perturbations of a Vandermonde matrix that preserve the Vandermonde structure.

\begin{lemma}[Operator norm bound for a Vandermonde inverse; equation 3.2 in \cite{gautschi1990stable}]\label{lem:vandermonde inverse norm bound}
Let $\alpha \in \R^k$ be entry-wise non-negative, and let $q(z) = \Prod_{i=1}^k (z-\alpha_i)$. Then 
 \[ \Norm{V_\alpha^{-1}}_{\infty} = \frac{\Abs{q(-1)}}{\min_i \Set{(1+\alpha_i)\Abs{q'(\alpha_i)}}}. \]
\end{lemma}

\begin{claim}
For roots $\alpha_1,\dotsc,\alpha_j$ satisfying $\Abs{\alpha_i - \alpha_j} \geq 	\zeta$, we have 
$\Norm{V_\alpha^{-1}}_{\infty} \leq 2^k/\zeta^{k-1}$. 
\end{claim}

\begin{proof}
We apply Lemma~\ref{lem:vandermonde inverse norm bound} and observe that $\Abs{q(-1)} \leq 2^k$ and $q'(\alpha_i) \geq \zeta^{k-1}$. 
\end{proof}

We define the derivative matrix of the Vandermonde matrix by interpreting each entry as the evaluation of a polynomial at a point, $[V_a]_{ij} = p_i(a_j)$,  where $p_i(t) = t^i$. Then $[V'_a]_{ij} = p'_i(a_j) = ia_j^{i-1}$.  

We'll now define the condition number of the system,  

\begin{equation}
\cond_{\infty}(a,b) \coloneqq \lim_{\eps \to 0} \sup_{
\substack{\Norm{\Delta a}_{\infty} \leq \eps\\ \Norm{\Delta b}_{\infty} \leq \eps}} 
\Setbar{\frac{\left\lVert \Delta x\right\rVert_{\infty}}{\eps}}{V(a+\Delta a)(x+\Delta x) = b + \Delta b}.
\end{equation}

We'll utilize a bound from \cite{bartels1992sensitivity}. After instantiating the theorem with the parameters relevant to our problem, the bound is the following:
\begin{theorem}[Theorem 2.2 of \cite{bartels1992sensitivity}]
\[ \cond_{\infty}(a,b) \leq \Norm{V_a^{-1}}_{\infty} + \Norm{V_a^{-1}V'_a\diag(x)}_{\infty}. \] \label{thm:vandermonde condition number}
\end{theorem}

\begin{lemma}\label{lem:well separated vandermonde cond number}
Let $\alpha \in [0,1]^k$, and let $w\in \R^k$ be a probability distribution over $[k]$. 
Let $\mu = V_{\alpha} w$. If $\zeta \leq \min_{i\neq j} \Abs{\alpha_i - \alpha_j}$, 
\[ \cond_{\infty}(\alpha, \mu) \leq (k+1)2^k/\zeta^{k-1}. \]
\end{lemma}

\begin{proof}
We observe that 
\begin{align*}
\Norm{V_\alpha^{-1}V'_\alpha\diag(w)}_{\infty} 
&\leq \Norm{V_\alpha^{-1}}_{\infty}\Norm{V'_\alpha\diag(w)}_{\infty}\\
&\leq 2^k/\zeta^{k-1}\Norm{V'_\alpha\diag(w)}_{\infty}\\
&= 2^k/\zeta^{k-1}\max_{i\in [k-2]} (i+1)\sum_{j=1}^k \Abs{\alpha_j^{i}w_j}\\
&\leq k2^k/\zeta^{k-1}\text{.}
\end{align*}
Applying the bound of  Theorem~\ref{thm:vandermonde condition number} gives the conclusion. 
\end{proof}

\begin{lemma} \label{lm: weight reconstruction}
Let $\alpha \in [0,1]^k$ and let $w\in \R^k$ be a probability distribution over $[k]$. 
Let $\mu = V_{\alpha}w$, and $\zeta \leq \min_{i\neq j} \Abs{\alpha_i - \alpha_j}$. 
Then $w' \coloneqq V^{-1}_{\tilde{\alpha}}\tilde{\mu}$ satisfies 
\[\Norm{w' - w}_{\infty} \leq \frac{(k+1)2^k}{\zeta^{k-1}}\max\left\{\Norm{\tilde{\alpha}-\alpha}_{\infty}, \Norm{\tilde{\mu}-\mu}_{\infty}\right\}.\] 
\end{lemma}

\begin{proof}
This follows from Lemma~\ref{lem:well separated vandermonde cond number} and the definition of the condition number.
\end{proof}

\begin{lemma} \label{lm: rounded weight reconstruction}
Given any weights $w'\in \R^k$ satisfying $\sum_{i=1}^k w'_i = 1$, the procedure $\textsc{RectifyWeights}(w')$ outputs in time $O(k)$ a weight vector $\tilde{w} \in [0,1]^k$ satisfying the following conditions
\begin{itemize}
    \item [(i)] $\sum_{i=1}^k\tilde{w}_i = 1$.
    \item [(ii)] $\Norm{\tilde{w} - w}_{\infty} \leq (k+1)\Norm{w' - w}_{\infty}$.
\end{itemize}
\end{lemma}

\begin{proof}
Note that in \ref{fig:algo2}, $I^{-}$ denotes the indices of the negative weights, and $I^{+}$ the positive weights. $W^{-}$ and $W^{+}$ denote the sums of the weights in the corresponding set of indices.

We'll now analyze $\tilde{w}$. First, note that we maintain property (i):
\begin{align*}
\sum_{i=1}^{k} \tilde{w}_i &= \sum_{i \in I^{+}} w'_i \Paren{1 + \frac{W^-}{W^{+}}}\\
&= W^+ \left( 1 + \frac{W^-}{W^+} \right) \\
&= W^+ + W^- = 1\
\end{align*}
Now we show that the weights are non-negative. Trivially, $\tilde{w}_i \geq 0$ for $i\in I^{-}$. For $i \in I^{+}$,
\begin{align*}
W^{+} &= 1 - W^-\\
    &= 1 + \abs{W^{-}}\\
    &\geq \abs{W^{-}}
\end{align*}
So $w'_i(1 + \frac{W^-}{W^{+}}) \geq 0$ if $i \in I^{+}$ as well.

We now prove (ii). We know that the true weights $w$ lie in $[0, 1]$, so increasing the negative weights to $0$ only 
moves them closer to their true values. Thus, we have $\abs{\tilde{w}_i - w_i} \leq \abs{ w'_i - w_i}$ for all 
$i \in I^-$. We observe that 
\[ \Abs{W^-} \leq \Norm{w'-w}_1 \leq k\Norm{w'-w}_{\infty} \] and then that 
\[ \Abs{\tilde{w}_i - w'_i} = \Abs{\underbrace{\Paren{w'_i/W^{+}}}_{\leq 1}W^-} \leq k \Norm{w'-w}_{\infty}. \] It follows that $\Norm{\tilde{w}-w'}_{\infty} \leq k\Norm{w'-w}_{\infty}$. Now we can apply the triangle inequality to get that 
\[ \Norm{\tilde{w}-w}_{\infty} \leq \Norm{\tilde{w}-w'}_{\infty} + \Norm{w'-w}_{\infty} \leq (k+1)\Norm{w'-w}_{\infty}. \]
To see that the runtime is $O(k)$ we observe we can compute $I^-$ and $I^+$ in linear time and likewise for $W^{-}$ and $W^{+}$. Each subsequent computation of $\tilde{w}_i$ takes constant time.
\end{proof}

\begin{corollary} \label{cor: full weight reconstruction}
Letting $\tilde{w} \in [0,1]^k$ be the output of $\textsc{RectifyWeights}(w')$ where $w'$ is as in Lemma~\ref{lm: weight reconstruction},  
\[ \Norm{\tilde{w} - w}_{\infty} \leq \frac{(k+1)^2 2^k}{\zeta^{k-1}}\max\left\{\Norm{\tilde{\alpha}-\alpha}_{\infty}, \Norm{\tilde{\mu}-\mu}_{\infty}\right\}. \]
\end{corollary}

\begin{proof}
Notice that the first equation in the linear system defining $w'$ is 
\[ \sum_{i=1}^k w'_i = \One^T w' = \tilde{\mu}_0 = 1.\]	Thus, $w'$ satisfies the hypothesis of Lemma~\ref{lm: rounded weight reconstruction} and the conclusion follows.
\end{proof}

\begin{algorithm}
\begin{algorithmic}[1]
\Procedure{RectifyWeights}{$w'$} 
\State $I^{-} \gets \Setbar{i}{w'_i < 0}$, \quad$I^{+} \gets \Setbar{i}{ w'_i \geq 0}$
\State $W^- \gets \sum_{i\in I^-} w'_i$,\quad $W^{+} \gets \sum_{i\in I^{+}} w'_i$
\For{$i=1,\dotsc,k$}
\State $\tilde{w}_i \gets \begin{cases}
    0 & \text{if }i\in I^{-}\\
    w'_i \left(1 + \frac{W^-}{W^{+}} \right) & \text{if }i\in I^{+}.
    \end{cases}$
\EndFor
\State Output $\tilde{w}$
\EndProcedure
\end{algorithmic} \caption{Algorithm \textsc{RectifyWeights}} \label{fig:algo2}
\end{algorithm}

\newpage
\appendix

\section{Deferred Proofs} \label{tossed}
\begin{proofof}{Lemma~\ref{lm: kernel characterization}}
(Part 1.)
By Equation~\eqref{eq: Hankel}, the rank of $\mathcal{H}_{k+1}$ for a $t$-coin 
 distribution is at most $t$,
and that implies that if $t\le k$, then $\mathcal{H}_{k+1}$ is singular. So consider a distribution $\mathcal{P}$ on
$[0,1]$ that has positive mass at $k+1$ points or more. Let $q\in\RR^{k+1}$ be a non-zero vector. We have
$$
q^\tpose \mathcal{H}_{k+1} q = \int_0^1 \left(\sum_{j=0}^k q_j \alpha^j\right)^2 \diff\mathcal{P}(\alpha) = \int_{0}^1 \hat{q}^2(\alpha) \diff \mathcal{P}(\alpha).
$$
There are at most
$k$ points in $[0,1]$ where the polynomial $\hat{q}$ evaluates to $0$, and the total $\mathcal{P}$ measure of those
points is less than $1$. Thus, $q^\tpose \mathcal{H}_{k+1} q > 0$, so $\mathcal{H}_{k+1}$ is positive definite. 

(Part 2.) Since $\mathcal{H}_{k+1}$ is symmetric, its kernel is spanned by $q$ s.t.\
$q^\tpose \mathcal{H}_{k+1} q = 0$. In order for the above integral to evaluate to zero over $\mathcal{P}$, we need that $\hat{q}^2(\alpha) = 0$ for each point  $\alpha \in \supp(\mathcal{P})$. As $\hat{q}$ is of degree $\leq k$, it is necessarily
a scalar multiple of $\prod_{i=1}^k (z-\alpha_i)$.
\end{proofof}

\begin{proofof}{Lemma~\ref{lem:pascal operator norm}}
We first observe that 
\[ \Pas = \begin{bmatrix} 
\binom{0}{0}{\binom{2k}{0}}^{-1} & \binom{1}{0}{\binom{2k}{0}}^{-1} & \dotsm & \binom{2k-1}{0}{\binom{2k}{0}}^{-1} & \binom{2k}{0}{\binom{2k}{0}}^{-1}\\
0 & \binom{1}{1}{\binom{2k}{1}}^{-1} & \dotsm & \binom{2k-1}{1}{\binom{2k}{1}}^{-1} & \binom{2k}{1}{\binom{2k}{1}}^{-1}\\
0 & 0 & \dotsm & \binom{2k-1}{2}{\binom{2k}{2}}^{-1} & \binom{2k}{2}{\binom{2k}{2}}^{-1}\\
\vdots & \vdots & \ddots & \vdots & \vdots \\
0 & 0 & \dotsm & 0 & \binom{2k}{2k}{\binom{2k}{2k}}^{-1}
\end{bmatrix} \] which can be factored to obtain 
\[ \Pas = \begin{bmatrix}\binom{2k}{0}^{-1} & 0 & 0 & \dotsm & 0\\
 0 & \binom{2k}{1}^{-1} & 0 & \dotsm & 0\\
 0 & 0 & \binom{2k}{2}^{-1} & \dotsm & 0\\
 \vdots & \vdots & \vdots & \ddots & \vdots \\
 0 & 0 & 0 & \dotsm & \binom{2k}{2k}^{-1}	
 \end{bmatrix} 
 \begin{bmatrix} 
\binom{0}{0} & \binom{1}{0} & \dotsm & \binom{2k-1}{0} & \binom{2k}{0}\\
0 & \binom{1}{1} & \dotsm & \binom{2k-1}{1} & \binom{2k}{1}\\
0 & 0 & \dotsm & \binom{2k-1}{2} & \binom{2k}{2}\\
\vdots & \vdots & \ddots & \vdots & \vdots \\
0 & 0 & \dotsm & 0 & \binom{2k}{2k}
\end{bmatrix}. \] Now 
\[ \Norm{\diag\Paren{\binom{2k}{0},\binom{2k}{1},\dotsc,\binom{2k}{2k}}^{-1}}_2 \leq 1. \] 
The Frobenius norm of the latter matrix is 
\[ \Paren{\sum_{j=0}^{2k} \sum_{i=0}^j \binom{j}{i}^2}^{1/2} 
= \Paren{\sum_{j=0}^{2k} \binom{2j}{j}}^{1/2} \leq \Paren{2k\binom{4k}{2k}}^{1/2} \leq (2^k4^{2k})^{1/2} \leq 6^k \] for $k \geq 2$. Using the sub-multiplicativity of the operator norm and the fact that the Frobenius norm upper bounds the operator norm, we get that 
$\Norm{\Pas} \leq 6^{k}$, as desired.
\end{proofof}


\section{Useful Theorems}\label{app: useful}

Consider two $n\times n$ Hermitian matrices $A$, $B$, with spectral decompositions
$A = \sum_{i=1}^n \kappa_i u_i u_i^\tpose$ and $B = \sum_{i=1}^n \lambda_i v_i v_i^\tpose$,
where the eigenvalues of both matrices are sorted in increasing order (i.e., $\kappa_1\le\kappa_2\le\cdots\le\kappa_n$
and $\lambda_1\le\lambda_2\le\cdots\le\lambda_n$). Also, let $P = B - A$ and
let $\rho_1\le\rho_2\le\cdots\le\rho_n$ be the eigenvalues of $P$ in increasing order.

\begin{theorem}[Weyl's inequality]\label{thm: Weyl}
For every $i\in\{1,2,\dots,n\}$,
$$
\kappa_i + \rho_1 \le \lambda_i \le \kappa_i + \rho_n.
$$
\end{theorem}

\begin{theorem}[Davis-Kahan $\sin\Theta$ theorem]\label{thm: Davis-Kahan}
Using the above definitions, let $i_0,i_1$ be integers such that $1\le i_0\le i_1\le n$,
and let
$$
g = \inf\{|\kappa - \lambda|:\ \kappa\in [\kappa_{i_0},\kappa_{i_1}]\wedge
     \lambda\in (-\infty,\lambda_{i_0-1}]\cup [\lambda_{i_1+1},+\infty)\},
$$
where we define $\lambda_0 = -\infty$ and $\lambda_{n+1} = \infty$.
Then,
$$
\left\|\sin \Theta(U,V)\right\|_F \le \frac{\|P\|_F}{g},
$$
where $U$ ($V$, respectively) is the $n\times i_1-i_0+1$ matrix whose columns are 
$u_{i_0},\dots,u_{i_1}$ ($v_{i_0},\dots,v_{i_1}$, respectively), $\Theta(U,V)$
is the $i_1-i_0+1\times i_1-i_0+1$ diagonal matrix whose $i$-th diagonal entry
is the $i$-th principal angle between the column spaces of $U$ and $V$, and
$\sin \Theta(U,V)$ is the diagonal matrix derived by applying the function $\sin$
entrywise to $\Theta(U,V)$. The same inequality holds if the Frobenius norm is
replaced by any orthogonally invariant norm, e.g., an operator norm $\|\cdot\|_{\op}$.
\end{theorem}

\begin{corollary}\label{cor: DK 1-dim}
Using the same definitions, 
$$
|u_1^\tpose v_1|\ge \sqrt{1 - \frac{\|P\|^2}{|\kappa_1-\lambda_2|^2}}.
$$
\end{corollary}

\begin{proof}
Take $i_0 = i_1 = 1$. By Theorem~\ref{thm: Davis-Kahan}, 
$|\sin \theta(u_1,v_1)|\le\frac{\|P\|}{|\kappa_1-\lambda_2|}$.
The corollary follows as
$|u_1^\tpose v_1| = |\cos \theta(u_1,v_1)| = \sqrt{1 - \sin^2 \theta(u_1,v_1)}$.
\end{proof}

\begin{theorem}[Courant-Fischer-Weyl min-max principle]\label{thm: Courant-Fischer-Weyl}
For every $i=1,2,\dots,n$,
\begin{eqnarray*}
\lambda_i & = & \min_{U\preceq\RR^n} \left\{\max_{x\in U} \left\{\frac{x^\tpose B x}{x^\tpose x}:\ x\ne 0\right\}:\ \dim(U) = i \right\} \\
& = & \max_{U\preceq\RR^n} \left\{\min_{x\in U} \left\{\frac{x^\tpose B x}{x^\tpose x}:\ x\ne 0\right\}:\ \dim(U) = n-i+1 \right\}.
\end{eqnarray*}
\end{theorem}

Let $C$ be an $m\times m$ Hermitian matrix with eigenvalues $\nu_1\le\nu_2\le\cdots\le\nu_m$,
where $m\le n$.
\begin{theorem}[Cauchy's interlacing theorem]\label{thm: Cauchy}
If $C = \Pi^* B \Pi$ for an orthogonal projection $\Pi$, then for all $i=1,2,\dots,m$ it holds that
$\lambda_i\le \nu_i\le \lambda_{n-m+i}$.
\end{theorem}

\begin{theorem}[Rouch\'e's theorem]\label{thm: Rouche}
Let $f$ and $g$ be two complex-valued functions that are holomorphic inside a region $R$
with a closed simple contour $\partial R$. If for every $x\in\partial R$ we have that $|g(x)| < |f(x)|$,
then $f$ and $f+g$ have the same number of zeros inside $R$, counting multiplicities.
\end{theorem}
\newpage
\bibliographystyle{plain}
\bibliography{2020.08.25}
\end{document}